\documentclass[conference]{IEEEtran}
\usepackage[utf8]{inputenc}

\IEEEoverridecommandlockouts

\usepackage{cite}
\usepackage{amsmath,amssymb,amsfonts}
\usepackage{algorithmic}
\usepackage{graphicx}
\usepackage{textcomp}
\usepackage{xcolor}
\usepackage{hyperref}
\usepackage{url}
\usepackage{algorithm}
\usepackage{subfig}
\usepackage{multirow}
\usepackage{mathtools}
\usepackage{amsthm}
\usepackage{booktabs}
\usepackage{acronym}
\usepackage{graphicx}   
\usepackage{marginnote} 
\reversemarginpar       

\usepackage{microtype}
\usepackage[capitalize,noabbrev]{cleveref}

\theoremstyle{plain}
\newtheorem{theorem}{Theorem}[section]

\newtheorem{lemma}[theorem]{Lemma}

\theoremstyle{definition}
\newtheorem{definition}[theorem]{Definition}

\theoremstyle{remark}

\DeclareMathOperator{\tv}{TV}
\DeclareMathOperator{\kl}{KL}

\def\BibTeX{{\rm B\kern-.05em{\sc i\kern-.025em b}\kern-.08em
    T\kern-.1667em\lower.7ex\hbox{E}\kern-.125emX}}

\acrodef{DP-SGD}{Differentially Private Stochastic Gradient Descent}
\acrodef{VAE}{Variational Autoencoder}
\acrodef{LCD-VAE}{Lipschitz-Constrained Decoder Variational Autoencoder}
\acrodef{FL}{Federated Learning}
\acrodef{DP}{Differential Privacy}
\acrodef{IID}{Independent and Identically Distributed}
\acrodef{Non-IID}{Non-Independent and Identically Distributed}
\acrodef{SGD}{Stochastic Gradient Descent}
\acrodef{GAN}{Generative Adversarial Network}
\acrodef{MSE}{Mean Squared Error}
\acrodef{ELBO}{Evidence Lower Bound}
\acrodef{KL}{Kullback-Leibler}
\acrodef{TV}{Total Variation}
\acrodef{FID}{Fr\'{e}chet Inception Distance}
\acrodef{ReLU}{Rectified Linear Unit}

\title{ALIGN-FL: Architecture-independent Learning through Invariant Generative component sharing in Federated Learning}

\author{
\IEEEauthorblockN{Mayank Gulati, Benedikt Groß, Gerhard Wunder}
\IEEEauthorblockA{\textit{Cybersecurity and AI Group, Freie Universität Berlin}\\
\{mayank.gulati, benedikt.gross, gerhard.wunder\}@fu-berlin.de}
\thanks{Accepted at 2025 International Conference on Cyber-Enabled Distributed Computing and Knowledge Discovery (CyberC).}
}
\begin{document}
\maketitle
\marginnote{\rotatebox{90}{%
  \footnotesize DOI: 10.1109/CyberC66434.2025.00025%
}}[-3cm]

\begin{abstract}
We present ALIGN-FL, a novel approach to distributed learning that addresses the challenge of learning from highly disjoint data distributions through selective sharing of generative components. Instead of exchanging full model parameters, our framework enables privacy-preserving learning by transferring only generative capabilities across clients, while the server performs global training using synthetic samples. Through complementary privacy mechanisms, DP-SGD with adaptive clipping and Lipschitz regularized VAE decoders, and a stateful architecture supporting heterogeneous clients, we experimentally validate our approach on MNIST and Fashion-MNIST datasets with cross-domain outliers. Our analysis demonstrates that both privacy mechanisms effectively map sensitive outliers to typical data points while maintaining utility in extreme Non-IID scenarios typical of cross-silo collaborations.

\end{abstract}

\begin{IEEEkeywords}
 Client-invariant Learning, \ac{FL}, Privacy-preserving Generative Models, \ac{Non-IID}, Heterogeneous Architectures
\end{IEEEkeywords}

\section{Introduction}
The increasing scale of machine learning models demands larger and more diverse training datasets, yet privacy concerns and data silos restrict direct data sharing. While \ac{FL} offers a promising direction by enabling collaborative model training without raw data exchange, it faces significant challenges when dealing with extreme data heterogeneity across clients. This challenge becomes particularly acute in scenarios where clients have completely non-overlapping data distributions - a common occurrence in real-world settings where organizations possess distinct data types, such as medical specialists with different imaging modalities or manufacturers with distinct product lines.

Traditional federated approaches like FedAvg \cite{mcmahan2017communication} and FedProx \cite{li2020federated} attempt to handle this heterogeneity through parameter averaging and regularization. However, these methods fundamentally assume some degree of overlap in client data distributions, making them ineffective when clients have completely disjoint data. In such cases, direct parameter or gradient sharing fails to capture and transfer knowledge effectively across these disparate domains.

We address these challenges through ALIGN-FL, a novel framework that reimagines federated learning through selective sharing of generative components. Instead of transmitting complete model parameters, we focus on transferring only the generative capabilities across clients, enabling consistent representation learning despite extreme differences in local data distributions. This design choice yields multiple benefits: efficient knowledge transfer, enhanced privacy preservation, and reduced communication overhead. Importantly, our approach achieves client-invariant unsupervised learning without requiring labeled data for alignment. 

ALIGN-FL is particularly well-suited for cross-silo \ac{FL} scenarios, where participants are typically organizations (such as hospitals, financial institutions, or manufacturers) with substantial computational resources but highly distinct data distributions. In these settings, our generative component sharing approach enables knowledge transfer while addressing the practical constraints of cross-silo environments: strict privacy requirements, heterogeneous system architectures, and completely non-overlapping data domains. Unlike cross-device \ac{FL} which focuses on efficiency for resource-constrained devices like considered in \cite{10745482}, our approach prioritizes robust knowledge sharing across organizational boundaries without compromising domain-specific expertise or privacy.
\subsection{Our contributions}
\begin{enumerate}
\item A privacy-preserving federated framework that enables knowledge sharing between clients with non-overlapping data distributions through selective sharing of generative components.
\item Two complementary privacy-preserving mechanisms: \ac{DP-SGD} with adaptive clipping and Lipschitz-constrained \ac{VAE} decoders, offering different privacy-utility trade-offs.
\item A stateful, architecture-independent client design that maintains local progress while contributing to global knowledge through synthetic data generation.
\item A comprehensive validation framework that evaluates both representation quality and privacy preservation, including targeted analysis of sensitive outlier handling.
\end{enumerate}
\section{Related Studies}
Generative models and their privacy face parallel challenges on multiple fronts. First, the increase in model size in large-scale machine learning models requires more high-quality training data \cite{hestness2017deep}; at the same time, privacy risks escalate as data dimensionality and model size increase \cite{dwork2014algorithmic}. Synthetic data generation has emerged as a promising approach to address this dual challenge. For instance, NVIDIA's Nemotron-4 340B model demonstrates the scalability of synthetic data, with over 98\% of its alignment data being synthetically generated \cite{adler2024nemotron}.

Data heterogeneity presents another significant challenge, particularly in \ac{FL} settings \cite{ye2023heterogeneous,li2020federated}. Recent work like FedCiR \cite{li2024fedcir} addresses \ac{Non-IID} data through classifier alignment by exchanging classifier outcomes and global parameter aggregation. While this approach shares our goal of handling heterogeneous data, our method takes a fundamentally different direction through generative modeling and synthetic data generation.

The practicality of synthetic data faces its own challenges, notably the phenomenon of model collapse \cite{briesch2023large, shumailov2024ai,futowards,gerstgrasser2024model}. This occurs when recursively generated data loses the tail of the original distribution, reducing output diversity. This challenge is particularly relevant for privacy considerations, as outliers in these distribution tails often represent the most privacy-sensitive data \cite{dwork2014algorithmic}. Recent theoretical work has established \ac{DP} guarantees for synthetic data generated by Lipschitz continuous \acp{VAE} \cite{gross2023differentially}. However, the privacy implications for model parameters themselves remain an open question, particularly in federated settings with heterogeneous data distributions.
\section{Problem Statement}
Consider a distributed system with $N$ clients, where each client $n$ has a local dataset $\mathcal{D}_n$ drawn from a client-specific distribution $\mathcal{P}_n(x)$. The distributions are highly non-overlapping, meaning:
\begin{equation}
\text{supp}(\mathcal{P}_i) \cap \text{supp}(\mathcal{P}_j) \approx \emptyset \quad \text{for } i \neq j
\end{equation}
This scenario is characteristic of cross-silo \ac{FL} where participating organizations possess substantial computational resources but completely distinct data domains. The objective is to learn a global model e.g., \ac{VAE} that not only captures the union of client distributions $\cup_{n=1}^N \mathcal{P}_n(x)$ for generation tasks but also learns a unified latent representation useful for downstream tasks. Each client maintains a local encoder-decoder pair $(q_{\phi_n}, p_{\theta_n})$, which must be aggregated into a global model $(q_{\phi}, p_{\theta})$ while preserving client privacy and data utility.

\subsection{Motivating Examples}

Consider two scenarios of non-overlapping data distributions across $5$ clients:
\begin{enumerate}
    \item \textbf{MNIST Partition}: Each client has unique digit pairs $\{0,1\}, \{2,3\}, \{4,5\}, \{6,7\}, \{8,9\}$, representing a clear partition of the digit space.
    \item \textbf{Domain-Specific Text}: Clients possess distinct text domains - scientific literature, legal documentation, medical records, financial reports, and technical manuals, each with unique vocabulary and semantic patterns.
\end{enumerate}
Both cases are illustrating the challenge of learning a unified representation while preserving domain privacy and generation quality.

\section{Background}
\subsection{Variational Autoencoders}
\acp{VAE} \cite{kingma2013auto} learn compressed latent representations through encoder-decoder pairs. The encoder $q_\phi(z|x)$ maps inputs to latent distributions $\mathcal{N}(\mu_\phi(x), \sigma^2_\phi(x))$, while the decoder $p_\theta(x|z)$ reconstructs data from latent samples. Training maximizes the \ac{ELBO}:
\begin{equation} \label{eq:elbo}
\mathcal{L}(\theta, \phi; x) = \mathbb{E}_{q_\phi(z|x)}[\log p_\theta(x|z)] - D_{\text{KL}}(q_\phi(z|x)||p(z))
\end{equation}

For a mini-batch of size $B$ with Gaussian assumptions, this simplifies to:
\begin{equation} \label{eq:loss}
\begin{split}
\mathcal{L}_{\text{VAE}} &= \frac{1}{B}\sum_{i=1}^B \Big(\|x_i - f_\theta(z_i)\|^2_2 \\
&\quad + \frac{1}{2}\sum_{j=1}^J(\sigma^2_{\phi,j} + \mu^2_{\phi,j} - \log\sigma^2_{\phi,j} - 1)\Big)
\end{split}
\end{equation}
where $z_i = \mu_\phi(x_i) + \sigma_\phi(x_i) \odot \xi_i$, with $\xi_i \sim \mathcal{N}(0,I)$.

\subsection{Privacy Mechanisms}
\subsubsection{Differential Privacy}
\ac{DP-SGD} \cite{abadi2016deep} provides privacy guarantees by clipping per-example gradients and adding calibrated noise:
\begin{equation}
\theta_{t+1} = \theta_t - \eta\left(\frac{1}{B}\sum_{i=1}^B \tilde{g}_i + \mathcal{N}(0, \sigma^2C^2\mathbf{I})\right)
\end{equation}
where $\tilde{g}_i$ represents clipped gradients, $C$ is the clipping norm, and $\sigma$ is the noise multiplier. Our implementation extends this with adaptive clipping based on client-specific gradient distributions.

\subsubsection{Lipschitz Continuity}
A function $f$ is $L$-Lipschitz continuous if for all inputs $x, y$:
\begin{equation}
|f(x) - f(y)| \leq L|x - y|
\end{equation}

\textit{Privacy Connection:} Lipschitz continuity naturally limits a model's sensitivity to individual training examples by bounding output changes. For an $L$-Lipschitz model, the maximum influence of any single training point is bounded by $L$, providing a natural privacy mechanism that we leverage in our \ac{LCD-VAE} approach. While not equivalent to formal \ac{DP}, this property offers complementary privacy benefits with different utility trade-offs (further discussed in \cref{privacy_intuition}).
\vspace{-5pt}
\begin{table}[htbp]
\caption{Data Distribution Across Clients}
\label{tab:data-splits}
\centering
\footnotesize
\begin{tabular}{ccc}
\toprule
Client & MNIST & Fashion-MNIST \\
\midrule
1 & \{0,1\}$^*$ & \{T-shirt, Trouser\}$^\dagger$ \\
2 & \{2,3\} & \{Pullover, Dress\} \\
3 & \{4,5\} & \{Coat, Sandal\} \\
4 & \{6,7\} & \{Shirt, Sneaker\} \\
5 & \{8,9\} & \{Bag, Boot\} \\
\bottomrule
\end{tabular}
\\[0.15cm]
\footnotesize $^*$Fashion-MNIST outliers, $^\dagger$MNIST outliers
\end{table}
\subsection{Assumptions}
\label{assumptions}
For privacy analysis, we intentionally introduced outlier points to selected clients requiring enhanced privacy protection as highlighted in \cite{lui2015outlier}. These outliers comprise $1-5\%$ of an individual client's local dataset. For simplicity, we restrict outliers to a single client in our study.
\begin{figure}[!htb]
\centering
    \includegraphics[width=0.99\columnwidth]{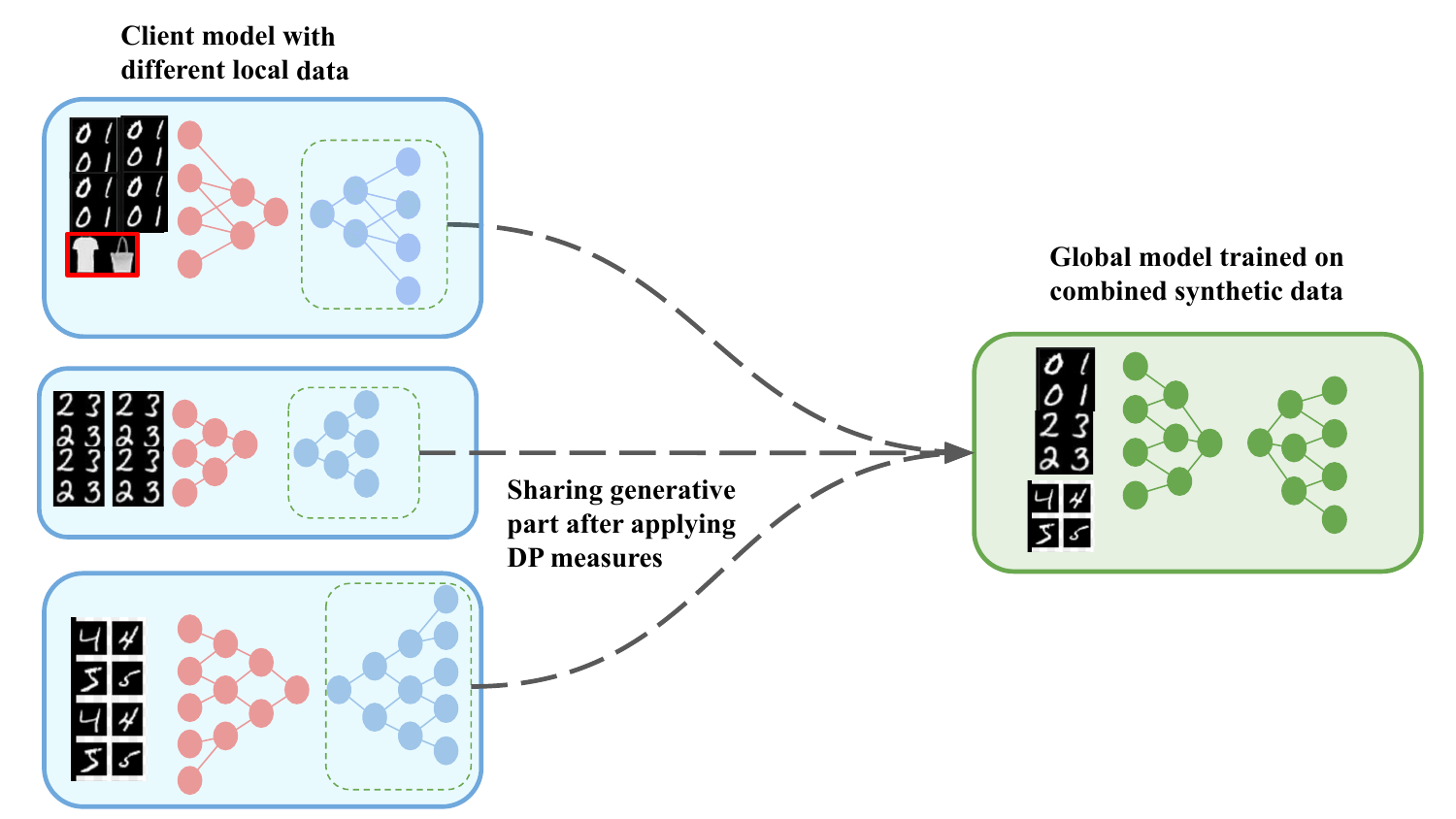}
    \caption{Distributed learning workflow for ALIGN-FL.}
    \label{fig:align-fl-workflow}
\end{figure}
\vspace{-5pt}

\vspace{-5pt}
\begin{table}[htbp]
\centering
\caption{ALIGN-FL Performance on Different Datasets}
\label{tab:align-fl-performance}
\footnotesize
\setlength{\tabcolsep}{4pt}
\begin{tabular}{llccc}
\toprule
Dataset & Method & FID$\downarrow$ & Acc\%$^{\dagger}$$\uparrow$ & F1\%$^{\dagger}$$\uparrow$ \\
\midrule
\multirow{3}{*}{MNIST} & No DP & 81.91 & 74.25 & 74.06\\
& DP-SGD & 174.01 & 51.75 & 51.8 \\
& LCD-VAE & 98.46 & 63.25& 60.01 \\
\midrule
\multirow{3}{*}{F-MNIST} & No DP & 117.97 & 69.75& 69.33 \\
& DP-SGD & 262.82& 52.5 & 50.51 \\
& LCD-VAE & 142.63 & 63& 62.51 \\
\bottomrule
\end{tabular}
\\[0.2cm]
\end{table}
\vspace{-5pt}

\vspace{-2pt}
\begin{table}[htbp]
\centering
\caption{Comparison of FL Algorithms on MNIST}
\label{tab:ablation}
\footnotesize
\setlength{\tabcolsep}{4pt}
\begin{tabular}{lccc}
\toprule
Algorithm & FID$\downarrow$ & Acc\%$^{\dagger}$$\uparrow$ & F1\%$^{\dagger}$$\uparrow$ \\
\midrule
FedAvg & 148.70 & 43.75 & 40.25 \\
FedProx & 150.14 & 38.50 & 34.07 \\
MOON-Base & 230 & 38.75 & 34.35 \\
\textbf{ALIGN-FL} & \textbf{81.91} &\textbf{ 74.25} & \textbf{74.06} \\
ContrastAvgSynth& 99.68 & 62.75 & 59.05 \\
ContrastSynthOnly& 89.64 & 52.5 & 50.66 \\
\bottomrule
\end{tabular}
\\[0.2cm]
\footnotesize $^{\dagger}$Classification on test set using multiclass logistic regression head
\end{table}
\vspace{-15pt}

\section{Our Approach: ALIGN-FL}
ALIGN-FL described in \cref{alg:align-fl} presents a novel approach to \ac{FL} that enables collaboration between clients with heterogeneous model architectures while maintaining privacy guarantees. The framework leverages the power of generative models (such as \acp{VAE} and \acp{GAN}) across distributed clients, where each client trains their local model using privacy-preserving techniques like \ac{DP-SGD} or Lipschitz regularization. Instead of sharing entire models, gradients or raw data, clients extract and share only their generative components (decoder in case of \ac{VAE}) with the central server, significantly reducing communication overhead while preserving architectural independence.

The server aggregates knowledge from diverse clients through a unique synthetic data generation process. After receiving generative components from clients, the server samples synthetic data from each client's generative model to create a combined dataset. This synthetic dataset captures the distributed knowledge across clients while maintaining privacy, as no raw data is shared. The server then trains a global model on this synthetic dataset, effectively learning from the collective knowledge of all clients without requiring architectural uniformity or compromising privacy guarantees.

This approach addresses several key challenges in \ac{FL}: it accommodates heterogeneous client architectures, preserves privacy through local \ac{DP} mechanisms and synthetic data generation, reduces communication overhead by sharing only generative components, and supports the asynchronous participation patterns typical in cross-silo environments. A crucial implementation consideration is that clients maintain their state between communication rounds to continuously enhance the global model through synthetic data generation at the server level, unlike vanilla \ac{FL} where client states are overwritten by global model updates. This stateful nature necessitates careful tracking of per-client privacy budgets when using \ac{DP-SGD}, ensuring rigorous privacy accounting while enabling asynchronous updates through parallel training.

\textbf{Key Departure from Traditional FL:} Unlike traditional FL that redistributes global parameters, ALIGN-FL centralizes the server model while clients contribute generative knowledge globally, avoiding parameter averaging degradation that plagues FedAvg and FedProx (\cref{tab:ablation}). Component $G_{c_i}^t$ is extracted from client model $M_{c_i}$ by copying only the generative module (e.g., VAE decoder parameters $p_{\theta_i}^t$) while keeping encoders client-side, preserving privacy, reducing communication overhead, and maintaining architecture independence.

\begin{algorithm}[tb]
\caption{ALIGN-FL}\label{alg:align-fl}
\begin{algorithmic}[1]
\STATE {\bfseries Input:} Number of clients $N$, local epochs $E$, communication rounds $T$
\STATE {\bfseries Initialize:} Server model ${M}_s$, client models ${M_{{c}_{i}}}$ with generative components ${G_{{c}_{i}}}$ for $i \in [1,N]$
\STATE {\bfseries Privacy Parameters:} Per-client DP budgets $(\epsilon_i, \delta_i)$ for $i \in [1,N]$
\FOR{round $t=1$ to $T$}
\STATE // Client Local Training Phase
\FOR{each client $i \in [1,N]$ in parallel}
\FOR{epoch $e=1$ to $E$}
\STATE Train ${M_{{c}_{i}}}$ on local data $\mathcal{D}_i$ using DP-SGD with $(\epsilon_i, \delta_i)$ or Lipschitz-regularized VAE
\ENDFOR
\STATE Extract generative component ${G^t_{{c}_{i}}}$ from ${M_{{c}_{i}}}$ 
\STATE Send ${G^t_{{c}_{i}}}$ to server
\ENDFOR
\STATE // Server Global Training Phase  
\STATE Initialize synthetic dataset $\mathcal{D}_s = \emptyset$
\FOR{each client $i \in [1,N]$}
\STATE Sample synthetic data $\mathcal{D}^{syn}_i \sim {G^t_{{c}_{i}}}$
\STATE $\mathcal{D}_s = \mathcal{D}_s \cup \mathcal{D}^{syn}_i$
\ENDFOR
\STATE Train server model ${M}_s$ on combined synthetic data $\mathcal{D}_s$
\ENDFOR
\STATE {\bfseries Return:} Final server model ${M}_s$
\end{algorithmic}
\end{algorithm}



We present several variants of client local training within ALIGN-FL framework, addressing different privacy scenarios alongside a baseline without privacy measures.
\subsection[DP-SGD on Generative Component of VAE only]{\ac{DP-SGD} on Generative Component of \ac{VAE} only}
Unlike GS-WGAN \cite{chen2020gs}, which achieved privacy by applying gradient sanitization only to the generator, this selective approach fails for \acp{VAE} due to their tightly coupled encoder-decoder training dynamics. Our experiments show that applying \ac{DP-SGD} solely to the \ac{VAE}'s decoder yields poor privacy-utility trade-offs, manifesting in two ways: (1) direct decoder sampling fails to generate coherent samples for sensitive data, and (2) the encoder-decoder pipeline produces artifacts resembling training data outliers. Detailed experimental results supporting these findings are presented in Appendix~\ref{appendix:dp-sgd_decoder}.

\subsection[DP-SGD on Full VAE Architecture]{\ac{DP-SGD} on Full \ac{VAE} Architecture}
Inspired by adaptive clipping techniques in \cite{andrew2021differentially}, we employ \ac{DP-SGD} with adaptive clipping based on client-side update norm quantiles. Rather than implementing their server-client bit sharing mechanism for clip norm adjustment, we adopt a simpler approach where each client independently computes and applies its local median for gradient clipping, eliminating the need for hyperparameter tuning.

\subsection[Lipschitz-Constrained Decoder VAE (LCD-VAE)]{Lipschitz-Constrained Decoder \ac{VAE} (\ac{LCD-VAE})}
In this variant of client-side training, we augment the standard \ac{VAE} objective (\cref{eq:loss}) with a gradient penalty term to enforce Lipschitz continuity on the decoder. Given interpolated points $\tilde{z}$ in the latent space, the gradient penalty is formulated as:
\begin{equation}
\mathcal{L}_{\text{GP}} = \mathbb{E}_{\tilde{z}}\left[\left(|\nabla_{\tilde{z}} f_\theta(\tilde{z})|_2 - 1\right)^2\right]
\end{equation}
where $|\nabla_{\tilde{z}} f_\theta(\tilde{z})|_2$ computes the gradient norm with respect to the latent points. The complete training objective is:
\begin{equation}
\mathcal{L}_{\text{total}} = \mathcal{L}_{\text{VAE}} + \lambda_{\text{GP}}\mathcal{L}_{\text{GP}}
\end{equation}
While gradient penalties have been employed in \acp{GAN} literature primarily for training stability of discriminator \cite{arjovsky2017wasserstein, chen2020gs}, our formulation serves a distinct purpose: constraining the decoder's sensitivity. This restriction limits the decoder's capacity to reconstruct points from the latent space, particularly affecting outliers and potentially sensitive samples in the data distribution.
\begin{figure*}[!htb]
  \centering
  \footnotesize
  \subfloat[Without DP client training]{\includegraphics[width=0.31\textwidth]{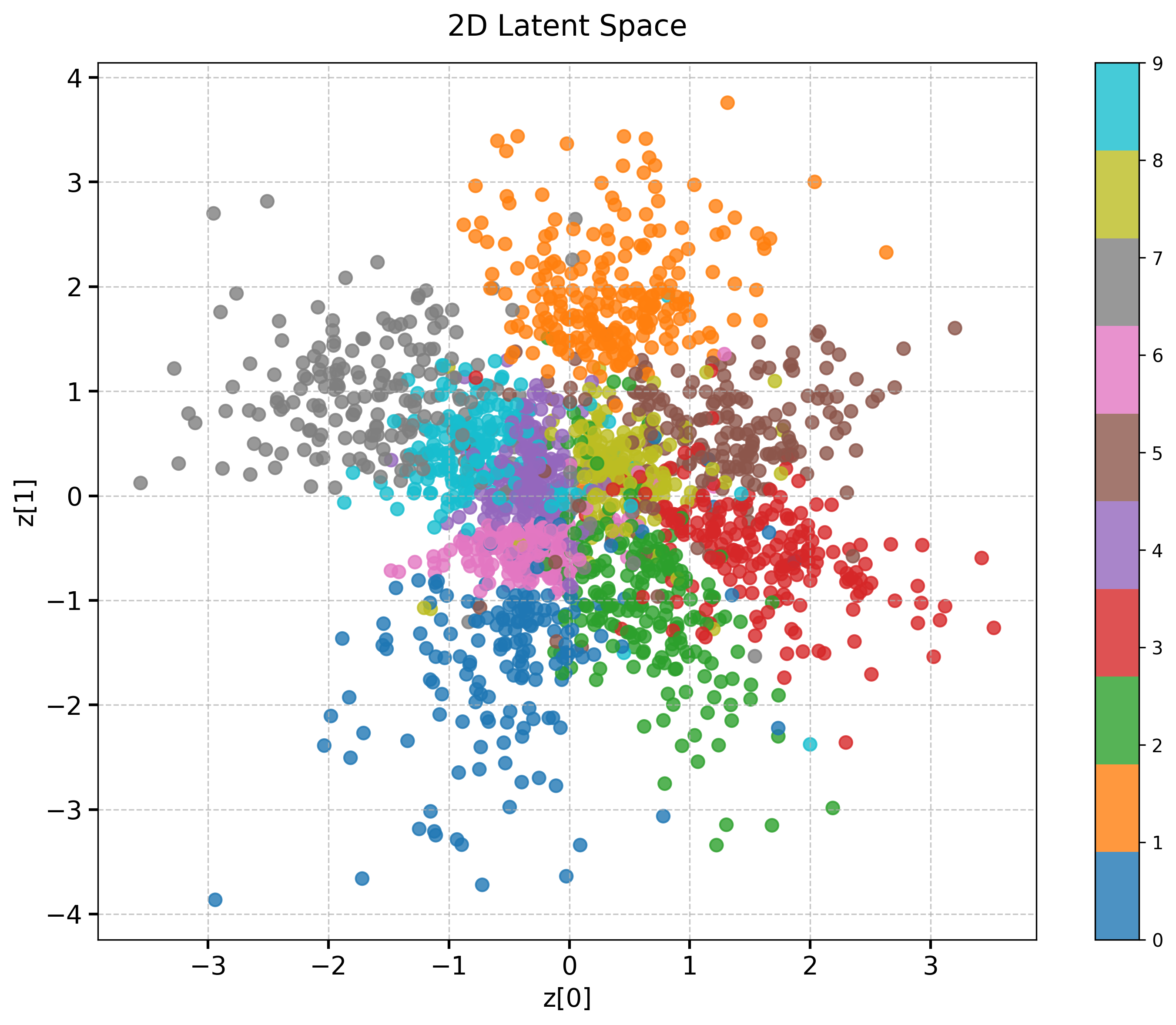}}
\hfill
  \footnotesize
  \subfloat[With DP-SGD $(\epsilon=10, \delta=10^{-5})$ client training]{\includegraphics[width=0.31\textwidth]{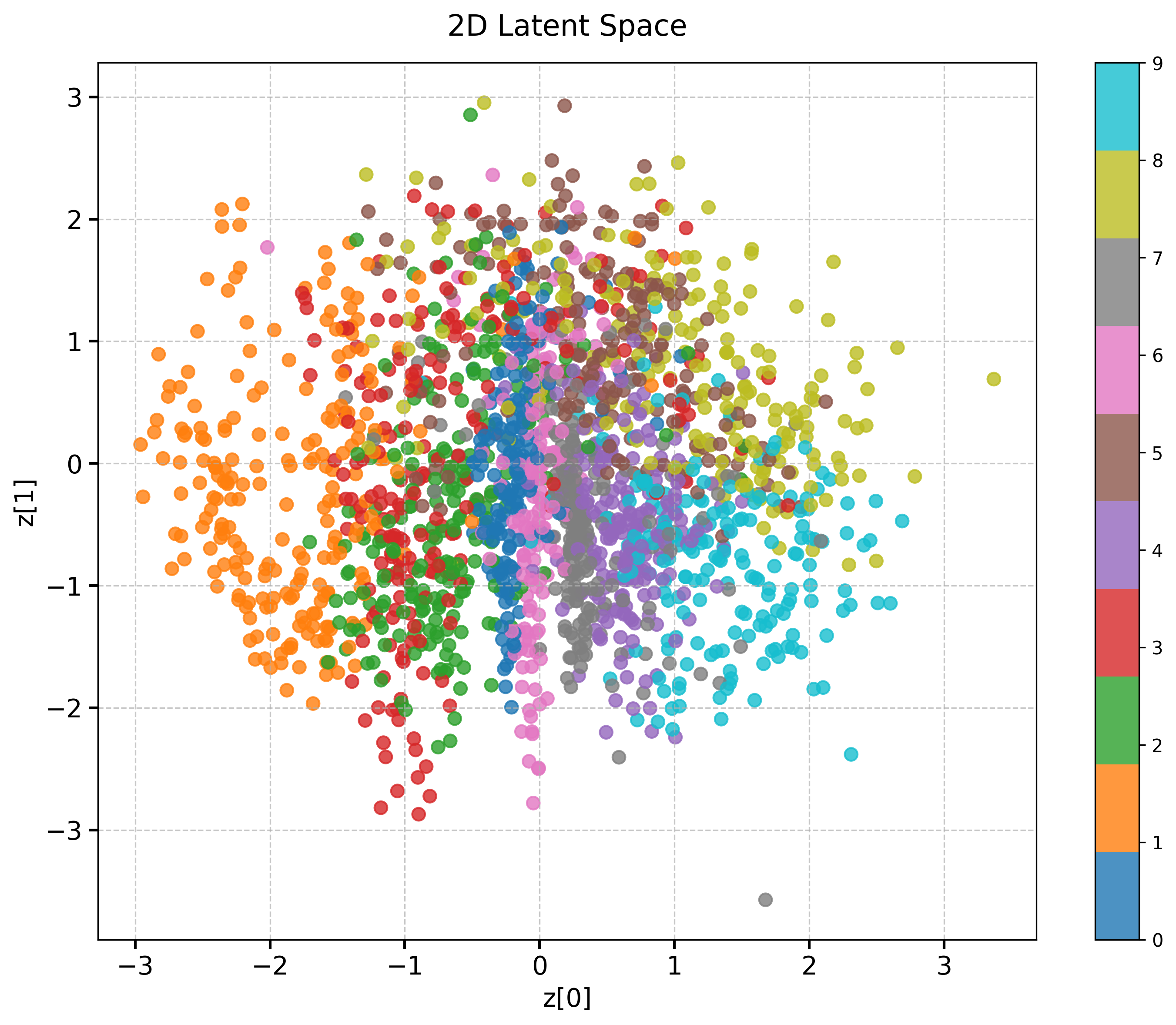}}
\hfill
 \subfloat[With LCD-VAE client training]{\includegraphics[width=0.31\textwidth]{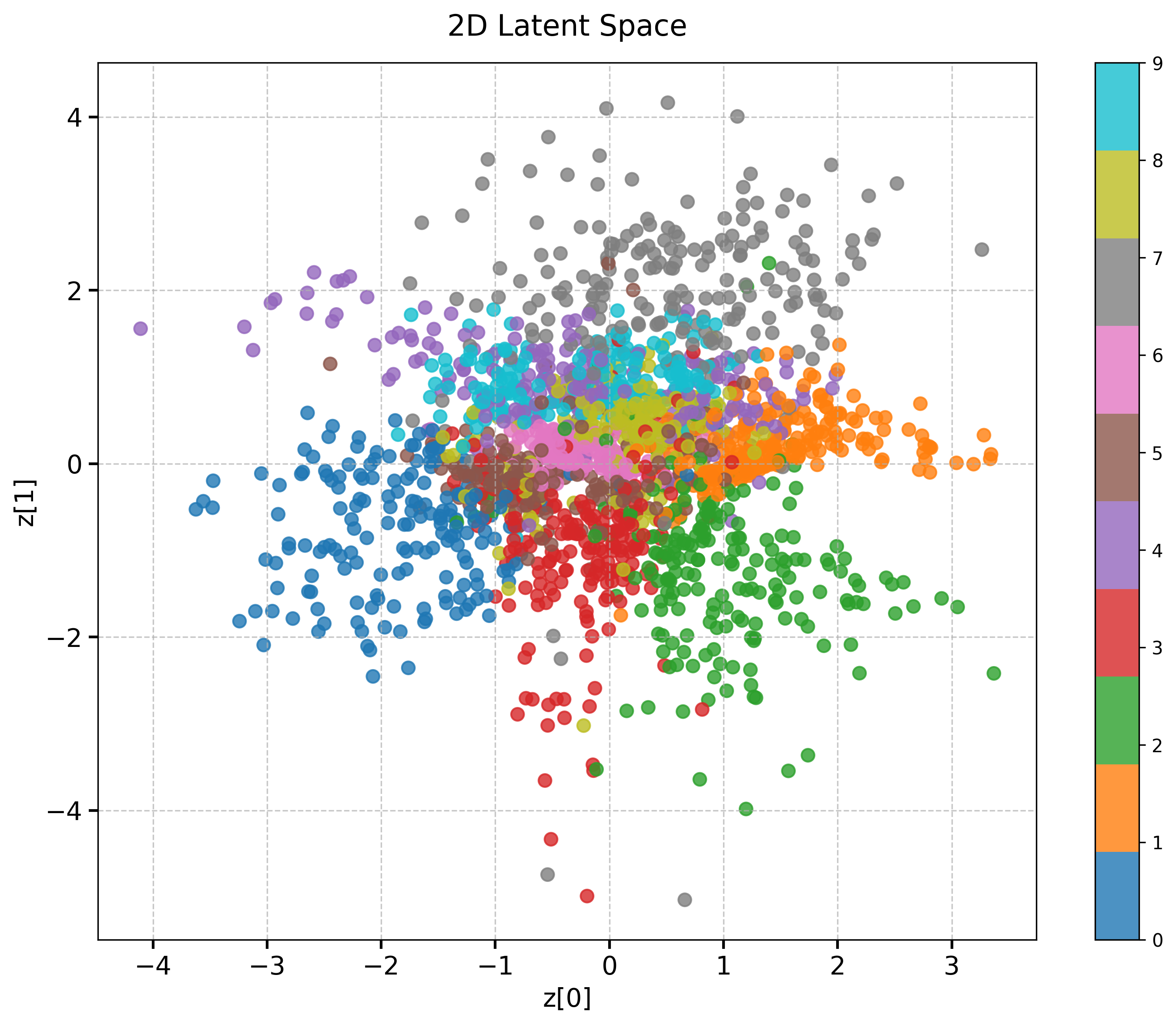}}
  \caption{ALIGN-FL trained global model's latent space representation on MNIST test set. Comparison across different privacy approaches shows varying degrees of cluster separation and organization.}

  \label{fig:ALIGN-FL_latent}
\end{figure*}
\section{Experiments}
We construct a controlled experimental setting to evaluate \ac{FL} under extreme data heterogeneity, establishing a foundation for client-invariant learning as discussed in \cite{li2024fedcir}. This approach aims to develop representations that are robust and meaningful across all clients despite their highly distinct and non-overlapping data distributions. The goal is to capture shared semantic structures even when raw data distributions are completely disjoint. The setup distributes MNIST and Fashion-MNIST datasets across $5$ clients with completely disjoint class distributions (please refer \cref{tab:data-splits}). To facilitate privacy analysis as highlighted in \cref{assumptions}, we introduce cross-domain outliers to Client 1's data - adding Fashion-MNIST samples to its MNIST distribution and vice versa.

Traditional approaches like FedAvg \cite{mcmahan2017communication} and FedProx \cite{li2020federated} encounter fundamental limitations when confronting extreme data heterogeneity scenarios. While FedAvg leverages parameter averaging and FedProx incorporates proximal regularization to harmonize local and global objectives, neither mechanism was designed to handle completely disjoint data distributions. The absence of overlapping support between client distributions creates overwhelming challenges for these methods. More recent approaches such as MOON \cite{li2021model}, PaDPaF\cite{almansoori2022padpaf}, which employ contrastive learning to encourage alignment between global positive representations while distancing them from previous model representations, similarly falter due to this lack of distributional support. This inadequacy is quantitatively evidenced by inferior accuracy, F1 scores, and \ac{FID} metrics \cref{tab:ablation} across these methods. Qualitative analysis in (see \cref{fig:federated_mnist} in Appendix) further illustrates this phenomenon, revealing disorganized global latent spaces where digit clusters significantly overlap, substantially degrading the global decoder's image generation quality. These consistent shortcomings across established methods underscore the critical need for novel approaches specifically engineered to address extreme data heterogeneity in \ac{FL} environments.

\section{Results}
We evaluate ALIGN-FL under three privacy-preserving strategies: baseline without privacy (ALIGN-FL No-\ac{DP}), \ac{DP} with \ac{DP-SGD}, and Lipschitz-constrained decoder \ac{VAE} (\ac{LCD-VAE}). Our experiments simulate the extreme data heterogeneity characteristic of cross-silo scenarios where organizations possess completely distinct data domains. These experiments on MNIST and Fashion-MNIST demonstrate that ALIGN-FL achieves structured latent representations across all variants, as shown in \cref{fig:ALIGN-FL_latent}. The baseline achieves the most structured representations, followed by \ac{LCD-VAE} and \ac{DP-SGD} respectively.

For quantitative evaluation, we attach a multi-class logistic regression layer to the global encoder trained using ALIGN-FL. The classifier and frozen encoder pipeline (architectural details in \cref{imp_details}) is trained on 80\% of the stratified testset and evaluated on the remaining 20\%. The performance metrics in \cref{tab:align-fl-performance} align with the qualitative observations, showing highest accuracy and F1-scores for the baseline, followed by LCD-VAE and DP-SGD. Notably, all variants significantly outperform random guessing (10\% for 10-class classification), indicating that meaningful representations are preserved despite privacy constraints. Despite the challenges of \emph{model collapse} often associated with synthetic data generation, these results indicate our approach maintains \emph{good discriminative power} while balancing privacy constraints.

The image generation quality follows a similar pattern, with \ac{LCD-VAE} producing sharper and more coherent samples than \ac{DP-SGD}. This visual assessment is validated by \ac{FID} scores in \cref{tab:align-fl-performance}, where \ac{LCD-VAE} achieves better scores by preserving local structure through Lipschitz constraints. Detailed visual comparisons of generated samples are provided in Appendix (\cref{fig:gen_quality_fmnist,fig:gen_quality_mnist}).

Analysis of image reconstruction dynamics in \cref{fig:privacy_mapping} reveals how both \ac{DP-SGD} variants and \ac{LCD-VAE} handle sensitive outlier points. Both mechanisms map outliers to typical points in the dataset, demonstrating their privacy-preserving capabilities—a behavior theoretically predicted by \cref{lemma:1}, which establishes that standard \acp{VAE} create problematic encoder mappings with minimal overlap between inputs.. When compared directly, \ac{LCD-VAE} tends to produce higher quality outputs with more semantically appropriate mappings to typical data points, consistent with Theorem A.6's prediction that structural constraints naturally reduce model sensitivity without requiring explicit noise injection. This suggests \ac{LCD-VAE} achieves a better privacy-utility trade-off, although establishing a formal connection between Lipschitz constraints and \ac{DP} guarantees remains an open theoretical challenge. In contrast, our baseline experiments with the No-\ac{DP} variant show it both reconstructs and generates outlier points during sampling, clearly indicating its lack of privacy protection (examples provided in Appendix \cref{fig:no_dp_base}).
\begin{figure*}[!htb]
  \centering
  \footnotesize
  \subfloat[Ground truth image]{\includegraphics[width=0.31\textwidth]{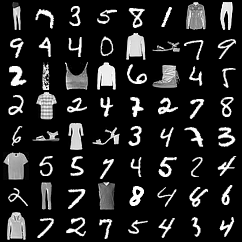}}
\hfill
  \footnotesize
  \subfloat[With DP-SGD $(\epsilon=10, \delta=10^{-5})$ client training]{\includegraphics[width=0.31\textwidth]{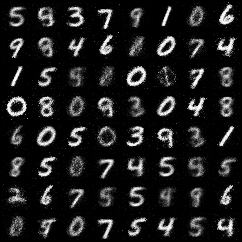}}
\hfill
 \subfloat[With LCD-VAE client training]{\includegraphics[width=0.31\textwidth]{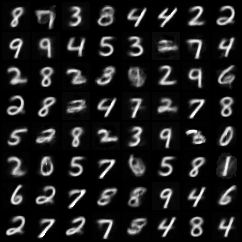}}
  \caption{Privacy-preserving outlier mapping in ALIGN-FL reconstruction. (a) Ground truth shows clear outliers (Fashion-MNIST samples). (b) DP-SGD maps outliers to semantically similar MNIST digits but with noise artifacts. (c) LCD-VAE produces cleaner mappings of outliers to typical data points, demonstrating how Lipschitz constraints naturally preserve semantic structure while protecting sensitive information.}
  \label{fig:privacy_mapping}
\end{figure*}
\vspace{-5pt}

\subsection{Ablation Study}
To understand the contribution of different components in ALIGN-FL, we conducted a series of ablation experiments by systematically evaluating architectural variations of our framework. We examined four distinct configurations to isolate the effects of different learning mechanisms:

\begin{itemize}
    \item \textbf{MOON-Base}: Standard MOON algorithm implementing contrastive loss on clients with weight averaging aggregation on the server.
    \item \textbf{ALIGN-FL}: Our proposed approach without differential privacy measures, employing client training without contrastive loss and server model training on synthetic data without weight averaging.
   \item \textbf{ContrastAvgSynth}: Client-side contrastive loss (MOON-style) combined with both weight averaging and synthetic data training on the server.
    \item \textbf{ContrastSynthOnly}: Client training with contrastive loss (MOON-style) and server-side synthetic data training only (no weight averaging).
\end{itemize}

\cref{tab:ablation} presents the comparative performance across these configurations. The results demonstrate that ALIGN-FL achieves the best performance across all metrics (FID: 81.91, Accuracy: 74.25\%, F1-Score: 74.06\%), while traditional federated learning approaches and MOON-Base show significantly inferior results.

This performance hierarchy reveals several critical insights: First, parameter averaging fundamentally undermines model quality in extremely heterogeneous settings. When clients possess entirely non-overlapping data distributions, the global model's knowledge acquired through synthetic data training is effectively diminished through averaging. Second, the comparison between ContrastAvgSynth and ContrastSynthOnly presents an interesting trade-off - while ContrastSynthOnly achieves better image quality (FID: 89.64 vs 99.68), ContrastAvgSynth demonstrates superior classification performance (Accuracy: 62.75\% vs 52.5\%, F1-Score: 59.05\% vs 50.66\%), likely due to its server-side synthetic data training after averaging. Despite this improvement, both contrastive approaches still underperform compared to standard ALIGN-FL, indicating that contrastive loss at clients is less effective than ALIGN-FL's approach when dealing with extreme heterogeneity.

The superior performance of configurations that prioritize synthetic data training without averaging (ALIGN-FL being the prime example) empirically validates our approach's core premise - that generative component sharing and server-side training on synthetic data provide a more effective mechanism for knowledge transfer across disjoint data distributions than traditional parameter averaging techniques or client-side contrastive learning methods.
\section{Discussion}
Our implementation of ALIGN-FL demonstrates several key insights into privacy-preserving \ac{FL} for non-overlapping distributions. The theoretical foundations in \cref{lemma:1}--\cref{lemma:federated_lipschitz_generation} establish the privacy guarantees of our approach, while empirical analysis reveals why conventional methods fail under extreme heterogeneity and how our selective generative component sharing addresses these limitations.

\subsection{Privacy Mechanisms and Theoretical Insights}
ALIGN-FL introduces two distinct privacy-preserving mechanisms, each with different theoretical foundations and practical implications. \ac{DP-SGD} provides formal $(\epsilon,\delta)$-\ac{DP} guarantees through bounded perturbations, while our Lipschitz-constrained \ac{VAE} approach achieves privacy through structural limitations on the decoder's capacity. As shown in \cref{lemma:1}, standard \acp{VAE} naturally create encoder mappings with minimal overlap between inputs, which can lead to privacy vulnerabilities when shared directly. Our \ac{LCD-VAE} approach addresses this through gradient penalties that limit the decoder's sensitivity to individual training examples.

Recent work \cite{gross2023differentially} analyzed privacy for Lipschitz \acp{VAE} using Bayesian posterior sampling for single-shot generation, while \ac{FL} requires composition analysis across multiple rounds.

The privacy transformation visualized in \cref{fig:privacy_mapping} empirically validates \cref{lemma:lipschitz}, which establishes that imposing constraints on model outputs naturally reduces sensitivity to individual examples. Both mechanisms effectively map outliers to semantically similar typical samples, preserving essential structure while protecting sensitive details.

\subsection{Limitations and Enhancement Possibilities}
Despite ALIGN-FL's strong performance, several limitations should be acknowledged. First, the quality gap between privacy-preserving and non-private synthetic data widens significantly for more complex data domains beyond our experimental setup. This limitation could be addressed through domain-specific architectural enhancements like convolutional VAEs or more advanced generative models.

Second, while our empirical results demonstrate \ac{LCD-VAE}'s privacy-preserving properties, establishing formal $(\epsilon,\delta)$-\ac{DP} guarantees for Lipschitz-constrained models remains an open challenge. Future work could develop tighter theoretical connections between Lipschitz constants and \ac{DP} parameters.

Third, our approach introduces additional computational burden on the server for synthetic data generation and training, which may become significant with very large numbers of clients. Hierarchical aggregation strategies could help address this scalability challenge while maintaining privacy guarantees.

\subsection{Practical Deployment Considerations}
For real-world cross-silo deployments, ALIGN-FL offers several practical advantages. The architecture independence enables organizations to maintain specialized model designs optimized for their specific data modalities and computational resources. This flexibility is particularly valuable in cross-domain scenarios like healthcare collaborations across specialties, financial institutions with diverse data types, and various supply chains \cite{10338891}.

For domains with highly complex data distributions, client-side pretraining with dedicated privacy budgets offers a promising enhancement. Clients could initially develop robust local representations before engaging in federated updates, reducing required communication rounds while preserving privacy. This staged approach may be particularly valuable for specialized domains where domain knowledge can inform initial model architectures.

Our stateful client design allows organizations to maintain their local progress throughout the federated process, enabling continuous improvement without the knowledge regression that occurs in traditional parameter averaging approaches when faced with non-overlapping distributions.

\section{Conclusion}
ALIGN-FL introduces a novel framework for privacy-preserving \ac{FL} across non-overlapping data distributions through selective sharing of generative components. Our approach significantly outperforms traditional \ac{FL} methods in extreme heterogeneity scenarios, achieving substantially higher accuracy, F1-scores, and generation quality while effectively preserving privacy.

By decoupling architectural choices between participants and focusing on generative capabilities rather than full parameter sharing, our framework enables knowledge transfer in settings where conventional approaches fundamentally break down. The experimental results demonstrate that both privacy mechanisms - \ac{DP-SGD} and \ac{LCD-VAE} - effectively balance utility and privacy, with \ac{LCD-VAE} often achieving better utility at comparable privacy levels.

The framework's versatility extends beyond our experimental setup, with potential applications in cross-organizational collaborations where data sharing is restricted but knowledge sharing is valuable. Examples include medical institutions collaborating across specialties, manufacturers sharing quality control knowledge across different product lines, or financial institutions detecting fraud patterns across diverse financial instruments.

Future research directions include extending ALIGN-FL to more complex domains through advanced generative architectures, establishing formal privacy bounds for Lipschitz-constrained models, and developing adaptive mechanisms for privacy budget allocation across heterogeneous clients. By providing a principled approach to \ac{FL} that addresses the fundamental challenges of heterogeneity, privacy, and communication efficiency, ALIGN-FL opens new possibilities for privacy-preserving collaborative learning in real-world environments characterized by extreme data heterogeneity. While ALIGN-FL demonstrates strong empirical performance, formal convergence guarantees remain to be established and are a valuable direction for future work.\vspace{-6pt}
\section*{Acknowledgments}
Mayank Gulati, Benedikt Groß, and Gerhard Wunder were supported by the Federal Ministry of Education and Research of Germany (BMBF) within “6G-RIC: 6G Research and Innovation Cluster,” under project identification number 16KISK025. Gerhard Wunder is also supported by the BMBF joint project “UltraSec: Security Architecture for UWB-based Application Platform,” project identification number 16KIS1682, and by the German Science Foundation (DFG) within priority program SPP 2378: “ResNets: Resilience in Connected Worlds” under grant WU 598/12-1.

\vspace{-6pt}
\bibliographystyle{ieeetr}
\bibliography{references}

@inproceedings{futowards,
  title={Towards Theoretical Understandings of Self-Consuming Generative Models},
  author={Fu, Shi and Zhang, Sen and Wang, Yingjie and Tian, Xinmei and Tao, Dacheng},
  booktitle={Forty-first International Conference on Machine Learning},
    year={2024}
}

@inproceedings{
gerstgrasser2024model,
title={Is Model Collapse Inevitable? Breaking the Curse of Recursion by Accumulating Real and Synthetic Data},
author={Matthias Gerstgrasser and Rylan Schaeffer and Apratim Dey and Rafael Rafailov and Tomasz Korbak and Henry Sleight and Rajashree Agrawal and John Hughes and Dhruv Bhandarkar Pai and Andrey Gromov and Dan Roberts and Diyi Yang and David L. Donoho and Sanmi Koyejo},
booktitle={First Conference on Language Modeling},
year={2024}}

@article{shumailov2024ai,
  title={AI models collapse when trained on recursively generated data},
  author={Shumailov, Ilia and Shumaylov, Zakhar and Zhao, Yiren and Papernot, Nicolas and Anderson, Ross and Gal, Yarin},
  journal={Nature},
  volume={631},
  number={8022},
  pages={755--759},
  year={2024},
  publisher={Nature Publishing Group UK London}
}

@inproceedings{gross2023differentially,
  title={Differentially private synthetic data generation via lipschitz-regularised variational autoencoders},
  author={Gro{\ss}, Benedikt and Wunder, Gerhard},
  booktitle={2023 IEEE Smart World Congress (SWC)},
  pages={1--8},
  year={2023},
  organization={IEEE}
}

@inproceedings{lui2015outlier,
  title={Outlier privacy},
  author={Lui, Edward and Pass, Rafael},
  booktitle={Theory of Cryptography: 12th Theory of Cryptography Conference, TCC 2015, Warsaw, Poland, March 23-25, 2015, Proceedings, Part II 12},
  pages={277--305},
  year={2015},
  organization={Springer}
}

@article{gopi2021numerical,
  title={Numerical composition of differential privacy},
  author={Gopi, Sivakanth and Lee, Yin Tat and Wutschitz, Lukas},
  journal={Advances in Neural Information Processing Systems},
  volume={34},
  pages={11631--11642},
  year={2021}
}

@article{chen2020gs,
  title={Gs-wgan: A gradient-sanitized approach for learning differentially private generators},
  author={Chen, Dingfan and Orekondy, Tribhuvanesh and Fritz, Mario},
  journal={Advances in Neural Information Processing Systems},
  volume={33},
  pages={12673--12684},
  year={2020}
}

@article{andrew2021differentially,
  title={Differentially private learning with adaptive clipping},
  author={Andrew, Galen and Thakkar, Om and McMahan, Brendan and Ramaswamy, Swaroop},
  journal={Advances in Neural Information Processing Systems},
  volume={34},
  pages={17455--17466},
  year={2021}
}

@article{kingma2013auto,
  title={Auto-encoding variational bayes},
  author={Kingma, Diederik P},
  journal={arXiv preprint arXiv:1312.6114},
  year={2013}
}

@inproceedings{abadi2016deep,
  title={Deep learning with differential privacy},
  author={Abadi, Martin and Chu, Andy and Goodfellow, Ian and McMahan, H Brendan and Mironov, Ilya and Talwar, Kunal and Zhang, Li},
  booktitle={Proceedings of the 2016 ACM SIGSAC conference on computer and communications security},
  pages={308--318},
  year={2016}
}

@inproceedings{arjovsky2017wasserstein,
  title={Wasserstein generative adversarial networks},
  author={Arjovsky, Martin and Chintala, Soumith and Bottou, L{\'e}on},
  booktitle={International conference on machine learning},
  pages={214--223},
  year={2017},
  organization={PMLR}
}

@article{li2024fedcir,
  title={Fedcir: Client-invariant representation learning for federated non-iid features},
  author={Li, Zijian and Lin, Zehong and Shao, Jiawei and Mao, Yuyi and Zhang, Jun},
  journal={IEEE Transactions on Mobile Computing},
  year={2024},
  publisher={IEEE}
}

@article{briesch2023large,
  title={Large language models suffer from their own output: An analysis of the self-consuming training loop},
  author={Briesch, Martin and Sobania, Dominik and Rothlauf, Franz},
  journal={arXiv preprint arXiv:2311.16822},
  year={2023}
}

@article{ye2023heterogeneous,
  title={Heterogeneous federated learning: State-of-the-art and research challenges},
  author={Ye, Mang and Fang, Xiuwen and Du, Bo and Yuen, Pong C and Tao, Dacheng},
  journal={ACM Computing Surveys},
  volume={56},
  number={3},
  pages={1--44},
  year={2023},
  publisher={ACM New York, NY, USA}
}

@inproceedings{mcmahan2017communication,
  title={Communication-efficient learning of deep networks from decentralized data},
  author={McMahan, Brendan and Moore, Eider and Ramage, Daniel and Hampson, Seth and y Arcas, Blaise Aguera},
  booktitle={Artificial intelligence and statistics},
  pages={1273--1282},
  year={2017},
  organization={PMLR}
}

@article{li2020federated,
  title={Federated optimization in heterogeneous networks},
  author={Li, Tian and Sahu, Anit Kumar and Zaheer, Manzil and Sanjabi, Maziar and Talwalkar, Ameet and Smith, Virginia},
  journal={Proceedings of Machine learning and systems},
  volume={2},
  pages={429--450},
  year={2020}
}

@article{beutel2020flower,
  title={Flower: A friendly federated learning research framework},
  author={Beutel, Daniel J and Topal, Taner and Mathur, Akhil and Qiu, Xinchi and Fernandez-Marques, Javier and Gao, Yan and Sani, Lorenzo and Li, Kwing Hei and Parcollet, Titouan and de Gusm{\~a}o, Pedro Porto Buarque and others},
  journal={arXiv preprint arXiv:2007.14390},
  year={2020}
}

@article{dwork2014algorithmic,
  title={The algorithmic foundations of differential privacy},
  author={Dwork, Cynthia and Roth, Aaron and others},
  journal={Foundations and Trends{\textregistered} in Theoretical Computer Science},
  volume={9},
  number={3--4},
  pages={211--407},
  year={2014},
  publisher={Now Publishers, Inc.}
}

@article{hestness2017deep,
  title={Deep learning scaling is predictable, empirically},
  author={Hestness, Joel and Narang, Sharan and Ardalani, Newsha and Diamos, Gregory and Jun, Heewoo and Kianinejad, Hassan and Patwary, Md Mostofa Ali and Yang, Yang and Zhou, Yanqi},
  journal={arXiv preprint arXiv:1712.00409},
  year={2017}
}

@article{adler2024nemotron,
  title={Nemotron-4 340B Technical Report},
  author={Adler, Bo and Agarwal, Niket and Aithal, Ashwath and Anh, Dong H and Bhattacharya, Pallab and Brundyn, Annika and Casper, Jared and Catanzaro, Bryan and Clay, Sharon and Cohen, Jonathan and others},
  journal={arXiv preprint arXiv:2406.11704},
  year={2024}
}

@misc{wandb,
title = {Experiment Tracking with Weights and Biases},
year = {2020},
note = {Software available from wandb.com},
url={https://www.wandb.com/},
author = {Biewald, Lukas},
}

@article{almansoori2022padpaf,
  title={PaDPaf: Partial disentanglement with partially-federated GANs},
  author={Almansoori, Abdulla Jasem and Horv{\'a}th, Samuel and Tak{\'a}{\v{c}}, Martin},
  journal={arXiv preprint arXiv:2212.03836},
  year={2022}
}

@inproceedings{li2021model,
  title={Model-contrastive federated learning},
  author={Li, Qinbin and He, Bingsheng and Song, Dawn},
  booktitle={Proceedings of the IEEE/CVF conference on computer vision and pattern recognition},
  pages={10713--10722},
  year={2021}
}

@ARTICLE{10745482,
  author={Gulati, Mayank and Zandberg, Koen and Huang, Zhaolan and Wunder, Gerhard and Adjih, Cedric and Baccelli, Emmanuel},
  journal={IEEE Access}, 
  title={TDMiL: Tiny Distributed Machine Learning for Microcontroller-Based Interconnected Devices}, 
  year={2024},
  volume={12},
  number={},
  pages={167810-167826},
  doi={10.1109/ACCESS.2024.3492921}}

@INPROCEEDINGS{10338891,
  author={Gulati, Mayank and Dadkhah, Narges and Groß, Benedikt and Wunder, Gerhard and Glavonjic, Jovan and Pavlovic, Aleksandar and Tomcic, Aleksandar},
  booktitle={2023 Fifth International Conference on Blockchain Computing and Applications (BCCA)}, 
  title={BETA-FL: Blockchain-Event Triggered Asynchronous Federated Learning in Supply Chains}, 
  year={2023},
  volume={},
  number={},
  pages={130-135},
  doi={10.1109/BCCA58897.2023.10338891}}
\appendix
\section{Appendix}
\subsection{Source code} \url{https://github.com/MakGulati/ALIGN-FL}.
\vspace{-2pt}
\subsection{Model Architecture \& Training Details}
\label{imp_details}
The \ac{VAE} uses symmetric architecture with encoder (512$\rightarrow$256$\rightarrow$128) mapping to 2D latent space ($\mu$, $\log \sigma^2$) and decoder (128$\rightarrow$256$\rightarrow$512), both using \ac{ReLU} activations and sigmoid output. All models use Adam optimizer ($\text{lr}=0.001$), $10$ federated rounds, batch size $128$, $10$ local epochs, with $5\%$ outlier samples.

\textbf{ALIGN-FL Configuration:} \ac{LCD-VAE} uses $\lambda_{\text{GP}}=1$ and $\lambda_{\text{KL}}=1$ based on \cref{fig:dp-lcd-utility}. Server training: $10$ epochs, batch=128, $5000$ global samples. \ac{DP-SGD}: $\varepsilon=10$, $\delta=10^{-5}$ based on \cref{fig:dp-sdg-utility}.

\textbf{Baselines:} FedAvg uses identical parameters. FedProx adds $\mu=1$. MOON uses $\text{temperature}=0.5$, $\mu=2$, with $5$ previous checkpoints as negatives.

\textbf{Evaluation:} Classification uses 80\% test data for training logistic regression head, 20\% for evaluation. Metrics: accuracy, F1-score, \ac{FID}.

\subsection{Implementation}
The experiments used the Flower framework \cite{beutel2020flower} with Weights \& Biases \cite{wandb} tracking. The privacy accounting follows \cite{gopi2021numerical}.

\begin{figure}[htp!]
    \centering
    \resizebox{0.51\textwidth}{!}{%
        \includegraphics{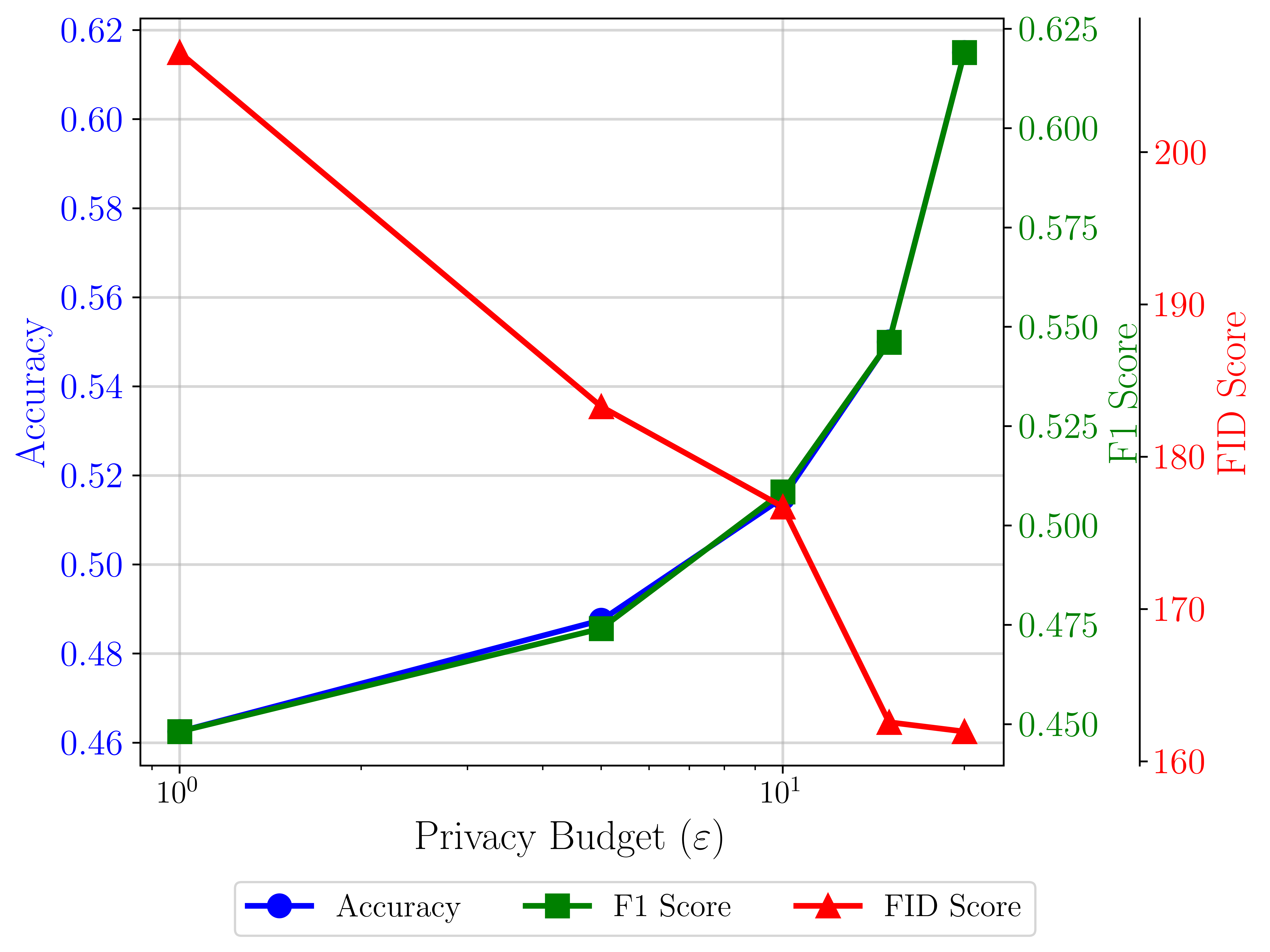}%
    }
    \caption{Privacy-Utility Tradeoff for different metrics in \ac{DP-SGD} for MNIST.}
    \label{fig:dp-sdg-utility}
\end{figure}
\begin{figure}[htp!]
    \centering
    \resizebox{0.5\textwidth}{!}{%
    \includegraphics[]{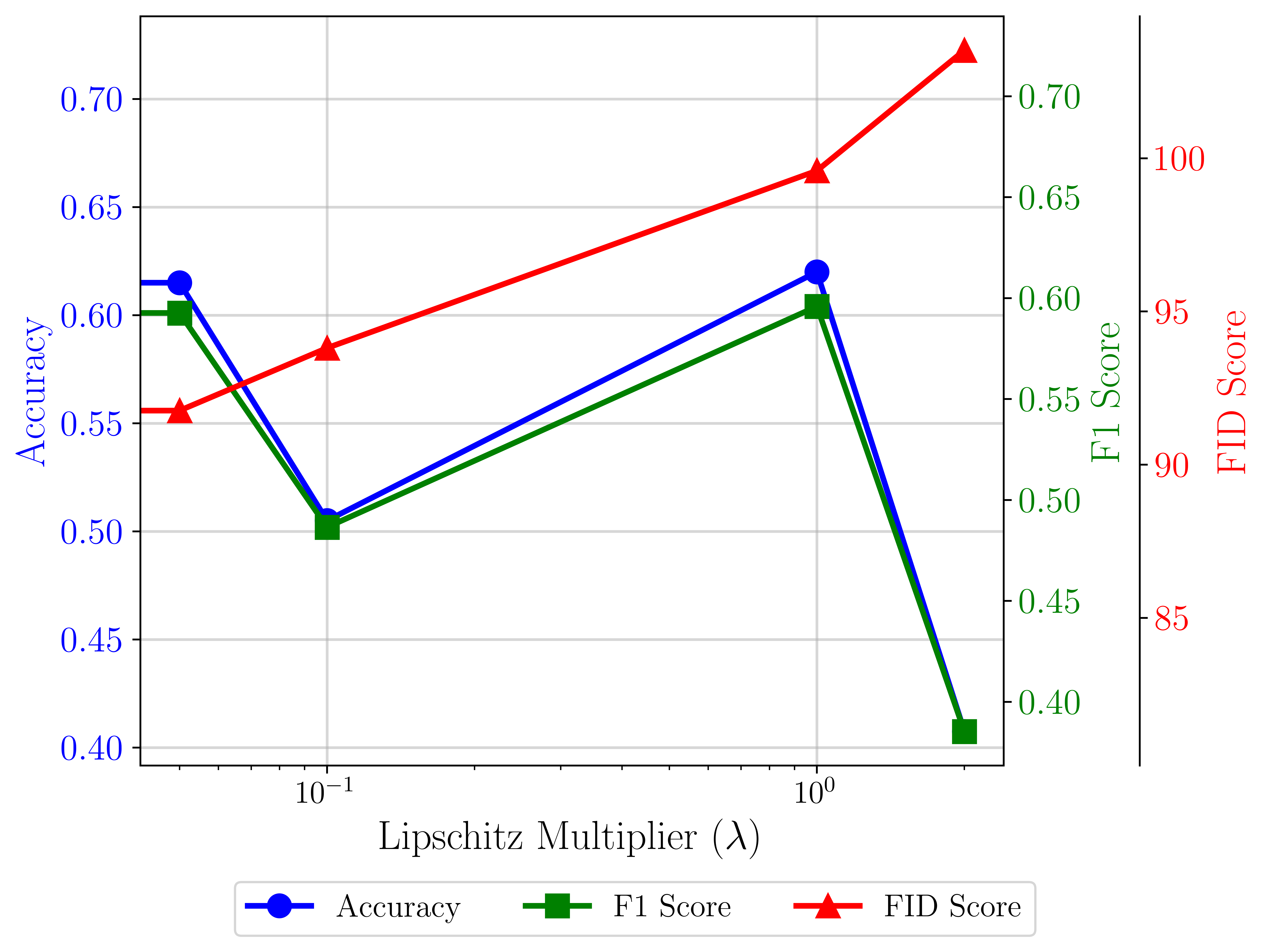}
    }
    \caption{Privacy-Utility Tradeoff for different metrics in \ac{LCD-VAE} for MNIST.}
    \label{fig:dp-lcd-utility}
\end{figure}

\begin{figure}[htp!]
    \centering
    \includegraphics[scale=1.5, width=.96\columnwidth]{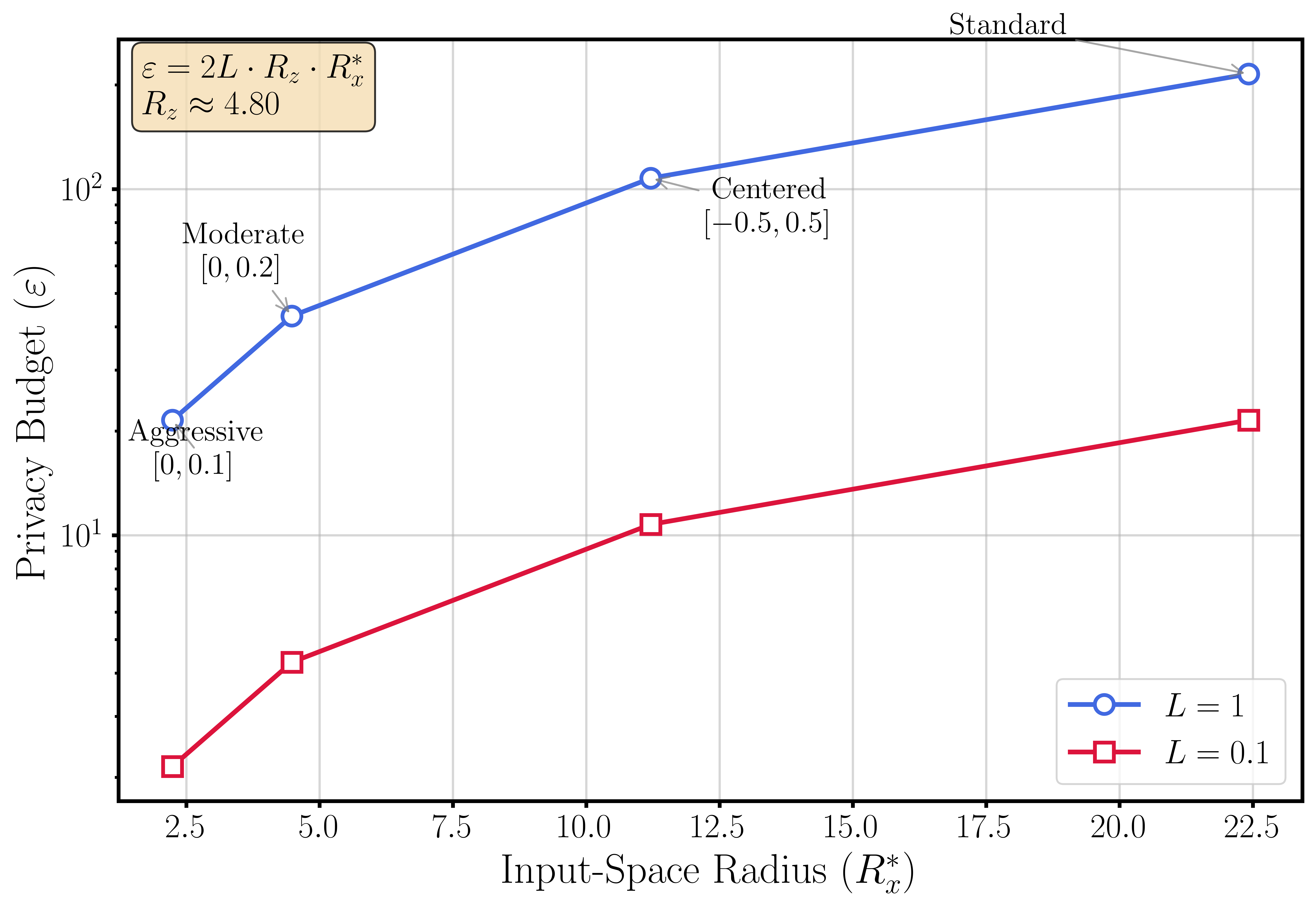}
    \caption{Privacy bound vs Input Space Radius with $\delta=10^{-5}$.}
    \label{fig:dp-vae-bound}
\end{figure}

\begin{figure}[!htb]
\vspace{-3pt}
\centering
\footnotesize
\setlength{\abovecaptionskip}{0pt}
\setlength{\belowcaptionskip}{-2pt}
\subfloat[FedAvg: Latent space]{\includegraphics[width=0.49\columnwidth]{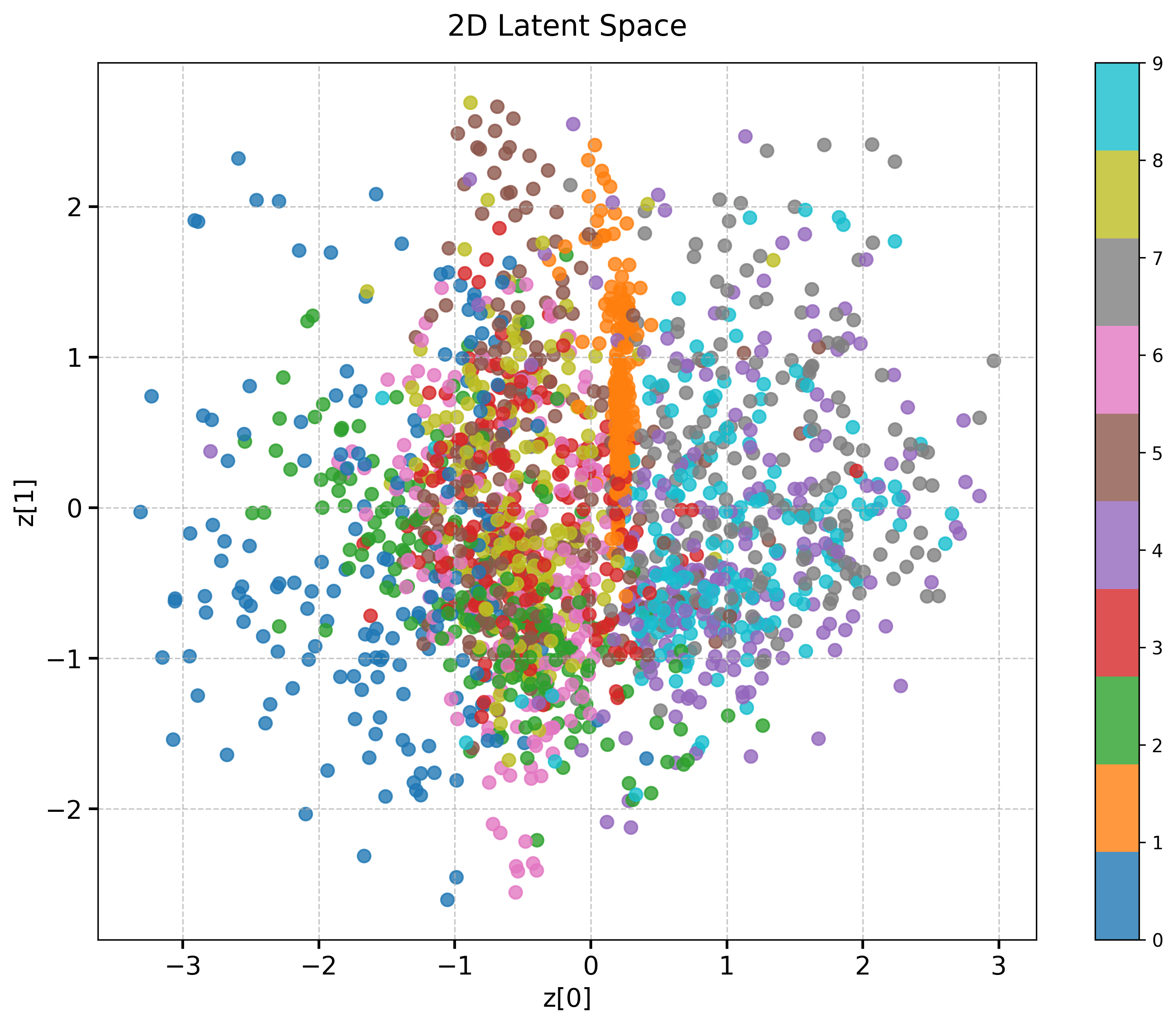}}%
\hfill%
\subfloat[FedAvg: Samples]{\includegraphics[width=0.42\columnwidth]{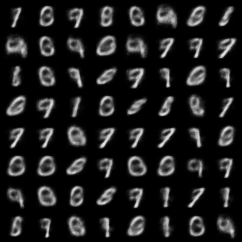}}\\%
\vspace{-2pt}%
\subfloat[FedProx: Latent space]{\includegraphics[width=0.49\columnwidth]{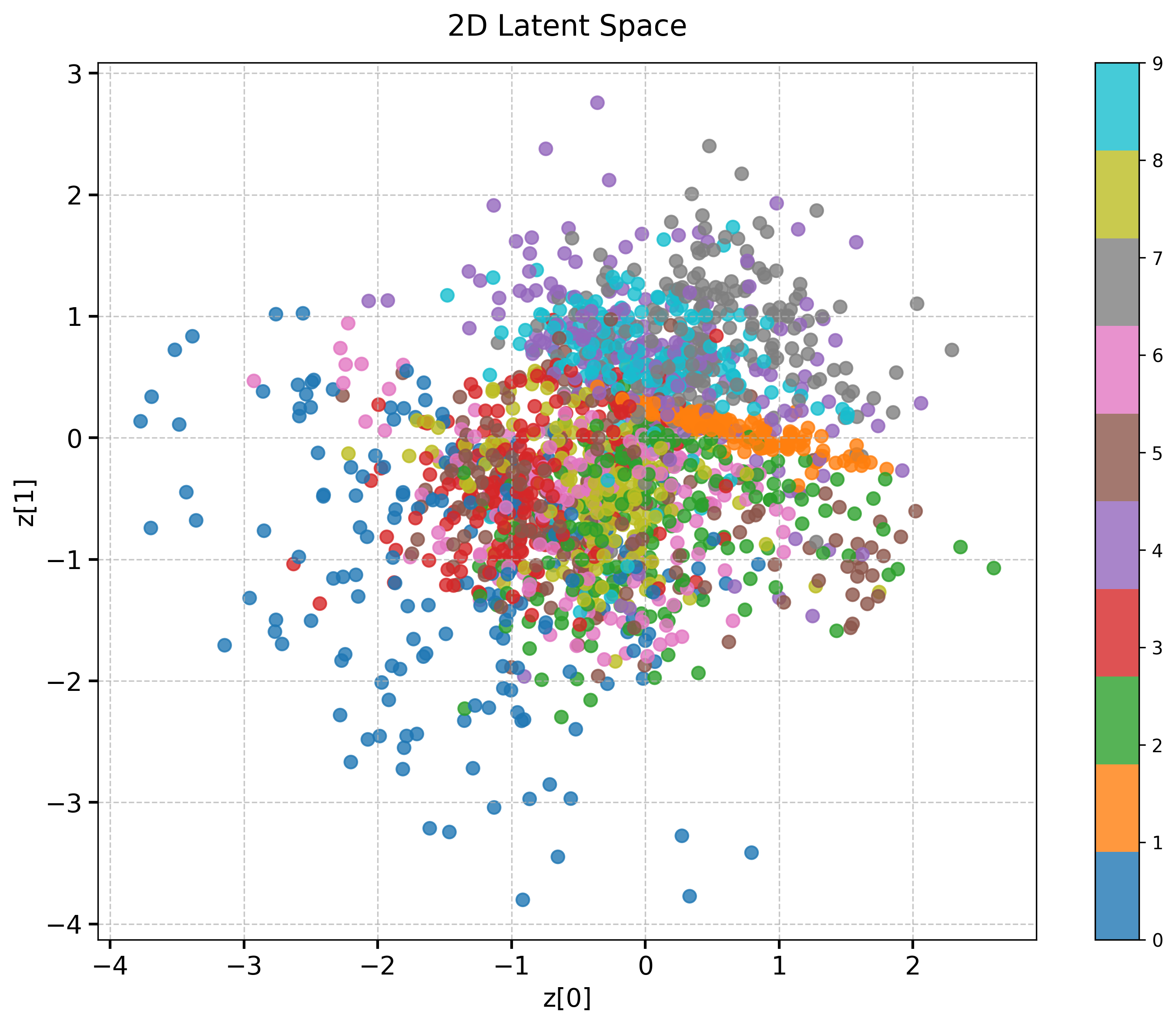}}%
\hfill%
\subfloat[FedProx: Samples]{\includegraphics[width=0.42\columnwidth]{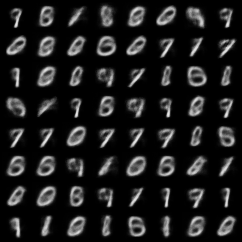}}%
\caption{Baseline results on MNIST. (a-b) FedAvg, (c-d) FedProx ($\mu=1$).}
\label{fig:federated_mnist}
\vspace{-4pt}
\end{figure}
\vspace{-5pt}
\subsection{Theoretical foundations}
\label{privacy_intuition}
\ac{DP} guarantees that  model outputs cannot reliably reveal individual training examples.

\begin{definition}[$(\epsilon, \delta)$-\ac{DP}] 
    A randomized mechanism $\mathcal{M}$ satisfies ($\epsilon$, $\delta$)-\ac{DP} if for all neighboring datasets $D$, $D'$ differing by one example and all possible outputs $S$:
\begin{equation}
P(\mathcal{M}(D) \in S) \leq e^\epsilon P(\mathcal{M}(D') \in S) + \delta.
\end{equation}
\end{definition}

We review two common \ac{DP} mechanisms. Let $D\in \{0,1\}^{|\mathcal{X}|}$ be the database encoding presence/absence of elements. Given query $f: \{0,1\}^{|\mathcal{X}|}\rightarrow \mathbb{R}^k$ and $\Delta f := \max_{D,D':\|D-D'\| \leq 1}\|f(D)-f(D')\|$, the Laplacian mechanism $\mathcal{M}_L(D,f,\epsilon) := f(D) + (Y_1,\cdots, Y_k)$ with $Y_i \sim \text{Lap}(\Delta f/\epsilon)$ provides $(\epsilon,0)$-\ac{DP}. The Gaussian mechanism $\mathcal{M}_G(D,f,\sigma) := f(D) + (Y_1,\cdots, Y_k)$ with $Y_i \sim \mathcal{N}(0, \sigma^2)$ satisfies the following:

\begin{theorem} \cite[Theorem A.1]{dwork2014algorithmic}
     For $\epsilon\in (0,1)$ and $c^2 > 2\ln (1.25/\delta)$, $\mathcal{M}_G(D,f,\sigma)$ is $(\epsilon, \delta)$-\ac{DP} if $\sigma \geq c\Delta f/\epsilon$.
\end{theorem}

Similar results for \ac{DP-SGD} are in \cite{abadi2016deep}. For Lipschitz constraints in \ac{VAE} training, models using \cref{eq:loss} follow:

\begin{lemma}\label{lemma:1}
    For nearly optimal $\phi$ and $x\neq x'$, there exists small $\epsilon>0$ such that 
    \[
        \mathbb{E}_{p^*(x)}\tv(q_{\phi}(z|x),q_{\phi}(z|x')) \geq 1-\epsilon.
    \]
\end{lemma}
\begin{proof}
    Since the decoder is deterministic, different $x$ cannot map from the same $z$:
    \[
        \tv(p_{\theta}(z|x),p_{\theta}(z|x')) = 1.
    \]
    By triangle inequality:
    \begin{align}
        &\tv(q_{\phi}(z|x),q_{\phi}(z|x'))\nonumber\\
        \geq& \tv(q_{\phi}(z|x),p_{\theta}(z|x')) - \tv(q_{\phi}(z|x'),p_{\theta}(z|x'))\nonumber\\
        \geq& \tv(p_{\theta}(z|x),p_{\theta}(z|x')) - \tv(q_{\phi}(z|x'),p_{\theta}(z|x')) \nonumber\\
        &- \tv(q_{\phi}(z|x),p_{\theta}(z|x))\nonumber\\
        = & 1- \tv(q_{\phi}(z|x'),p_{\theta}(z|x')) - \tv(q_{\phi}(z|x),p_{\theta}(z|x)).\label{eq:lm1}
    \end{align}
    
    Since \cref{eq:elbo} gives $\log p_{\theta}(x) - D_{\kl}(q_{\phi}(z|x)\|p_{\theta}(z|x))$, optimal parameters yield $\mathbb{E}_{p^*(x)} \tv(q_{\phi}(z|x), p_{\theta}(z|x))\leq \epsilon$ and $\mathbb{E}_{p^*(x')} \tv(q_{\phi}(z|x'), p_{\theta}(z|x'))\leq \epsilon$. Combined with \cref{eq:lm1}, the proof follows.
\end{proof}

This shows that Gaussian encoders have minimal overlap between inputs, particularly in high dimensions where the sphere covering decreases. Thus, standard \ac{VAE} training compromises privacy through explicit reconstruction.
By contrast, if the decoder is Lipschitz the Gaussian encoders must actually overlap.

    
\begin{lemma}[Lemma 4 from \cite{gross2023differentially} reformulated]\label{lemma:lipschitz}
    Let $q_{\phi_1}$ and $q_{\phi_2}$ be two different nearly optimal encoders trained on neighboring datasets $D$, $D'$ according to LCD-VAE, then for any $\epsilon>0$ there is small $L$ such that: $\mathbb{E}_{p^*(x)}\tv(q_{\phi_1}(z|x),q_{\phi_2}(z|x)) \leq 2\sqrt{\epsilon}.$
    
\end{lemma}
Here $\epsilon=2 L R_z(\delta) R_x$ where $R_x$ is the radius of data ball $B_x(R_x)=\{x\in\mathbb{R}^p,||x||\leq R_x\}$ and $R_z$ is the latent space radius such that ${P}\left(Z\in B_z^C(R_z)\right)\leq\delta$ shown in \cref{fig:dp-vae-bound}. This lemma guarantees one-shot privacy when i) latent space is sampled within $B_z(R_z)$, and ii) released data is sampled from the decoder's Gaussian output distribution. 

\begin{theorem}[Federated Lipschitz-Constrained Generation]
\label{lemma:federated_lipschitz_generation}
Consider trained decoders $f_{\theta_t}$ with Lipschitz constant $L$ across $T$ federated rounds. Suppose the data is released such that i) the latent space is sampled within $B_z(R_z)$, and, ii) the released data of size $|\mathcal{D}_s|$ is additionally sampled from the decoder's Gaussian output distribution. Then the data is $(2T|\mathcal{D}_s|\sqrt{\epsilon},T\delta_x)$-\ac{DP} where $\epsilon$. Recall that $\epsilon$ depends only on $L, R_x, R_z$ which are  hyperparameters, and $\delta_x$ is a (small) constant from the Monte-Carlo integration which decreases essentially with $O(1/\sqrt(n)$. 
\end{theorem}



This bound ensures that the cumulative sensitivity of generated data across federated rounds remains controlled, providing privacy protection that scales linearly with the number of communication rounds. The proof is omitted but follows basically the lines of \cite{gross2023differentially} correcting their argument which actually provides only a one-shot privacy guarantee in the release of the synthetic data. Note that the theorem does so by virtue of a slightly but essentially different sampling process and a worst-case privacy composition. 

\begin{figure}[htp]
\centering
\includegraphics[width=0.48\linewidth]{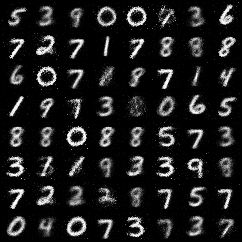}
\hfill
\includegraphics[width=0.48\linewidth]{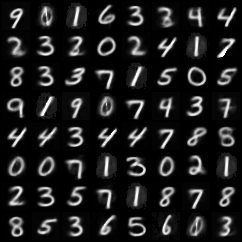}
\caption{Image generation quality of global decoder with ALIGN-FL. Left: \ac{DP-SGD} $(\epsilon=10, \delta=10^{-5})$. Right: \ac{LCD-VAE}}
\label{fig:gen_quality_mnist}
\end{figure}
\vspace{-15pt}
\begin{figure}
\centering
\includegraphics[width=0.48\linewidth]{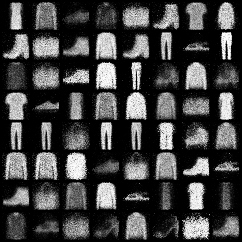}
\hfill
\includegraphics[width=0.48\linewidth]{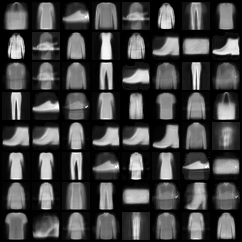}
\caption{Image generation quality of global decoder with ALIGN-FL. Left: \ac{DP-SGD} $(\epsilon=10, \delta=10^{-5})$. Right: \ac{LCD-VAE}}
\label{fig:gen_quality_fmnist}
\end{figure}
\vspace{-18pt}
\begin{figure}
\centering
\includegraphics[width=0.48\linewidth]{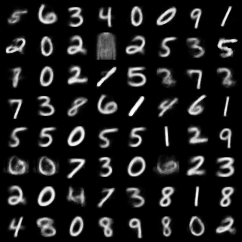}
\hfill
\includegraphics[width=0.48\linewidth]{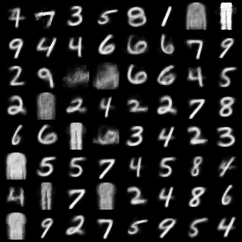}
\caption{No-\ac{DP} baseline. Left: Images sampled directly from the decoder. Right: Image reconstruction through auto-encoder.}
\label{fig:no_dp_base}
\end{figure}
\vspace{-15pt}
\begin{figure}
\centering
\includegraphics[width=0.48\linewidth]{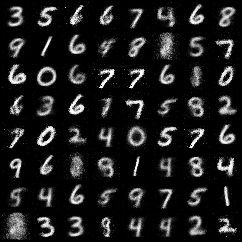}
\hfill
\includegraphics[width=0.48\linewidth]{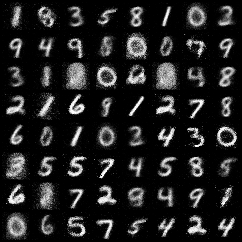}
\caption{Visualization of selective \ac{DP-SGD} $(\epsilon=10, \delta=10^{-5})$ application to \textbf{\ac{VAE} decoder only}. Left: Images sampled directly from the decoder. Right: End-to-end image reconstruction through auto-encoder.}
\label{appendix:dp-sgd_decoder}
\end{figure}
\end{document}